\documentclass[accepted]{uai2025} 
\usepackage{lmodern}  
\usepackage{textcomp} 
\usepackage{xspace}                                     

\usepackage[normalem]{ulem}

\usepackage{amsmath}
\usepackage{enumitem}
\usepackage{amsfonts, mathtools,amssymb, amsthm}
\usepackage{thm-restate}
\usepackage{dsfont} 
\usepackage{mleftright}

\newtheorem{theorem}{Theorem}[section]
\newtheorem{lemma}[theorem]{Lemma}

\newtheorem{definition}[theorem]{Definition}

\newtheorem{example}{Example}

\newtheorem{postulate}[theorem]{Postulate}

\newtheorem{principle}[theorem]{Principle}

\usepackage{algorithm}
\usepackage[noend]{algpseudocode}

\usepackage[dvipsnames,table]{xcolor}

\usepackage{multirow}
\usepackage{multicol}
\usepackage{chngpage} 

\usepackage{mfirstuc}

\usepackage{tikz}
\usetikzlibrary{arrows}
\usetikzlibrary{calc,decorations.pathmorphing,patterns}

\usepackage{hyperref}
\hypersetup{
    colorlinks=true,
    citecolor=ForestGreen,
    filecolor=black,
    linkcolor=VioletRed,
    urlcolor=blue
}

\usepackage{aliascnt}
\usepackage[numbered]{bookmark}
\usepackage[capitalise,nameinlink]{cleveref}

\Crefname{assumption}{Assumption}{Assumptions}
\Crefname{postulate}{Postulate}{Postulates}
\crefname{figure}{Figure}{Figures}
\crefname{principle}{Principle}{Principles}

\newcommand{\indep}{\perp\mkern-9.5mu\perp}

\newcommand{\ceq}{\stackrel{+}{=}}
\newcommand{\cle}{\stackrel{+}{\leq}}
\newcommand{\cge}{\stackrel{+}{\geq}}

\newcommand{\eps}{\ensuremath{\varepsilon}\xspace}

\newcommand{\defeq}{\vcentcolon=}

\newcommand{\domain}{\ensuremath{\mathcal{X}}\xspace} 

\newcommand{\kldiv}[2]{{\operatorname{KL}\mleft({#1 \mid\mid #2}\mright)}}

\newcommand\restr[2]{{
  \left.\kern-\nulldelimiterspace 
  #1 
  \vphantom{\big|} 
  \right|_{#2} 
  }}

\newcommand{\R}{\ensuremath{\mathbb{R}}\xspace}

\newcommand{\pdfsamp}{dual\xspace}
\newcommand{\cdfsamp}{cumulative dual\xspace}
\newcommand{\Pdfsamp}{\expandafter\capitalisewords\expandafter{\pdfsamp}}
\newcommand{\Cdfsamp}{\expandafter\capitalisewords\expandafter{\cdfsamp}}

\makeatletter

\newcommand{\Rom}[1]{\expandafter\@slowromancap\romannumeral #1@}

\makeatother

\algnewcommand{\LineComment}[1]{\Statex \(\triangleright\) #1}

\newcommand{\cX}{\mathcal{X}}

\newcommand{\cY}{\mathcal{Y}}

\newcommand{\Exp}{\mathbf{E}}

\newcommand{\bx}{\textbf{x}}

\newcommand{\bX}{\textbf{X}}

\DeclareMathSymbol{\ast}{\mathbin}{symbols}{"03}

\DeclareMathOperator*{\E}{\mathbf{E}}

\DeclareMathOperator{\pa}{{pa}}      
\DeclareMathOperator{\PA}{{PA}}      
\makeatletter
\newcommand*\bigcdot{\mathpalette\bigcdot@{.5}}
\newcommand*\bigcdot@[2]{\mathbin{\vcenter{\hbox{\scalebox{#2}{$\m@th#1\bullet$}}}}}
\makeatother

\usetikzlibrary{arrows,shapes,plotmarks,positioning}
\usetikzlibrary{decorations.markings}

\tikzset{>=stealth'} 
\tikzstyle{graphnode} = 
   [circle,draw=black,minimum size=22pt,text centered,text
     width=22pt,inner sep=0pt] 
\tikzstyle{var}   =[graphnode,fill=white]
\tikzstyle{vardashed}   =[graphnode,draw=gray,fill=white]
\tikzstyle{obs}   =[graphnode,fill=black,text=white]
\tikzstyle{obsgrey}   =[graphnode,draw=white,fill=lightgray,text=black]
\tikzstyle{par}    =[graphnode,draw=white,fill=red,text=black] 
 \tikzstyle{crucial} =[graphnode,draw=white,fill=yellow,text=black] 
\tikzstyle{fac}   =[rectangle,draw=black,fill=black!25,minimum size=5pt]
\tikzstyle{facprior} =[rectangle,draw=black,fill=black,text=white,minimum size=5pt]
\tikzstyle{edge}  =[draw=white,double=black,very thick,-]
\tikzstyle{blueedge}  =[draw=white,double=blue,very thick,-]
\tikzstyle{rededge}  =[draw=white,double=red,very thick,-]
\tikzstyle{prior} =[rectangle, draw=black, fill=black, minimum size=
5pt, inner sep=0pt]
\tikzstyle{dirprior} = [circle, draw=black, fill=black, minimum
size=5pt, inner sep=0pt]

\tikzstyle{dot_node}=[draw=black,fill=black,shape=circle]

\usepackage[textwidth=1.9cm]{todonotes}

\newcommand\independent{\protect\mathpalette{\protect\independenT}{\perp}}
\def\independenT#1#2{\mathrel{\rlap{$#1#2$}\mkern2mu{#1#2}}}

\usepackage[american]{babel}

\usepackage{natbib} 
\usepackage{mathtools} 
\usepackage{booktabs} 
\usepackage{tikz} 

\title{Toward Universal Laws of Outlier Propagation}

\author[1,*]{Aram~Ebtekar}
\author[2,3,*]{Yuhao~Wang}
\author[3]{Dominik~Janzing}
\affil[ ]{%
    \textsuperscript{1}Independent Researcher\qquad 
    \textsuperscript{2}National University of Singapore\qquad
    \textsuperscript{3}Amazon Research Tübingen, Germany\qquad\qquad\qquad\qquad
}

\begin{document}
\maketitle

{
    \def\thefootnote{*}\footnotetext{These authors contributed equally to this work.}
}

\begin{abstract}
When a variety of anomalous features motivate flagging different samples as \emph{outliers}, Algorithmic Information Theory (AIT) offers a principled way to unify them in terms of a sample's \emph{randomness deficiency}. Subject to the algorithmic Markov condition on a causal Bayesian network, we show that the randomness deficiency of a joint sample decomposes into a sum of randomness deficiencies at each causal mechanism. Consequently, anomalous observations can be attributed to their root causes, i.e., the mechanisms that behaved anomalously. As an extension of Levin's law of randomness conservation, we show that weak outliers cannot cause strong ones. We show how these information theoretic laws clarify our understanding of outlier detection and attribution, in the context of more specialized outlier scores from prior literature.
\end{abstract}

\section{Introduction}

Generically speaking, an \emph{outlier} is an anomalous or atypical event, suggesting possible interference and/or downstream effects. Anomaly detection plays a crucial role in business, technology, and medicine. Typical
use cases range from fraud detection in finance and online trading \citep{donoho2004early}, performance drops in manufacturing lines \citep{susto2017anomaly} and cloud computing applications \citep{Gan2021, ma2020automap, hardt2023petshop}, health monitoring in intensive care units \citep{maslove2016errors}, 
to explaining extreme weather and climate events \citep{Zscheischler2021}.
It has motivated a vast effort towards developing methodologies relevant to outlier analysis. For examples, we refer the reader to some of the early works in statistics and computer science \citep{freeman1995outliers, rocke1996identification, rousseeuw2003robust, aggarwal2017introduction}.

In complex systems, an anomaly may cause a large cascade of related anomalies \citep{panjei2022survey}. 
To mitigate them, it is not enough to merely 
\emph{detect} anomalies; we must also \emph{identify} which of the anomalies was the root cause \citep{root_cause_analysis, ikram2022root, li2022causal, hardt2023petshop, wang2023interdependent, wang2023incremental}. 
Thus, we implicitly face the \emph{counterfactual} question of what conditions could have been different to prevent the (usually undesired) anomalous event.

To render a complex system accessible to human understanding, we begin with a causal model of its relevant mechanisms, specifying not only their default behavior, but also their behavior under modifications called \emph{interventions}. Such a model should be modular in two respects. First, we may want to understand the \emph{causal pathway}, along which a perturbation of any part of the system propagates through its components until it generates the event. Second, we want to 
``blame'' some component(s) of the system, while acknowledging that others worked as expected. 

Causal Bayesian networks offer such a modular description, specifying causal relations via a directed acyclic graph (DAG) $G$ with random variables $X_1,\dots,X_n$ as nodes \citep{pearl2009causality, Spirtes1993}. Under the causal Markov condition \citep{pearl2009causality}, the joint distribution factorizes according to 
\begin{equation}\label{eq:markov}
P(X_1,\dots,X_n) = \prod_{j=1}^n P(X_j\mid \PA_j), 
\end{equation}
where $\PA_j$ denotes the parents of $X_j$ in $G$, i.e., its direct causes.
We will think of each conditional distribution $P(X_j\mid \PA_j)$ as an \emph{independent mechanism} of the system, which can in principle be changed or replaced without changing the others (see 2.1 and 2.2 in \cite{peters2017elements} for a historical overview).

Provided that \emph{most} of the mechanisms operate normally, we can blame an anomalous joint observation $x_1,\dots,x_n$ on a small number of mechanisms that we call \emph{root causes}, in alignment with \cite{Schoelkopf2021}'s 
``sparse mechanism shift hypothesis''.
\cite{root_cause_analysis} formalize additional concepts for attributing anomalies to mechanisms, providing a starting point for our theoretical development.

\subsection{Outlier scores from p-values}
\cite{root_cause_analysis} introduce what they call an Information Theoretic (IT) outlier score\footnote{While \cite{root_cause_analysis} use the natural base $e$, we use base $2$ logarithms to align with binary program lengths in algorithmic information theory. In effect, we express the outlier score in units of bits \citep{frank2005indefinite}.} via
\begin{equation}
\label{eq:itscore}
\lambda_\tau (x) := -\log P(\tau(X)\geq \tau(x)),
\end{equation}
where $x$ denotes an observation of the random variable $X$, and $\tau: {\cal X} \to \R$ is a feature statistic whose highest values we consider anomalous.

$\lambda_\tau$ can be viewed as a statistical test of the null hypothesis that $x$ was sampled from $P$: setting the base of the logarithm to $2$ yields the p-value $2^{-\lambda_\tau(x)}$. A small p-value (or large $\lambda_\tau$) corresponds to an unusual sample under $P$, which can be labeled an outlier. Since $x$ is a single observation, anomaly scoring thus reduces to classical hypothesis testing on a sample size of $1$ \citep{shen2020randomness,vovk2020non}\footnote{Note that \citet{tartakovsky2012efficient} discuss anomalies as change points over multiple observations, whereas we focus on anomalies confined to an individual observation.}. 

\subsection{Quantitative root cause analysis in causal networks}

In order to quantify the contribution of different upstream mechanisms to an anomalous observation, the arXiv version of \cite{root_cause_analysis} defines the \emph{conditional outlier scores}:
\[
\lambda_\tau (x_j\mid\pa_j) := 
-\log P(\tau(X_j)\ge \tau(x_j) \mid 
\PA_j=\pa_j).
\]

They demonstrate that this score can be interpreted as measuring the anomalousness of the mechanism's noise term,\footnote{This aligns in spirit with \cite{backtracking} where events are also backtracked by attributing them to the noise variables.}  whenever $\lambda_\tau(X_j\mid PA_j)$ is independent of $PA_j$.
The feature functions $\tau_j$ can be node-specific, which is essential when the variables $X_j$ have different characteristics (e.g., different dimensionality or data types). 
When the specific choice of $\tau$ is not crucial for the discussion, we will simplify notation by dropping the subscript.
The arXiv version of \cite{root_cause_analysis} extends this framework by introducing a joint outlier score through ``convolution'':
\begin{eqnarray}\label{eq:conv}
\lambda(x_1,\dots,x_n)
&:=& \sum_{j=1}^n \lambda(x_j\mid\pa_j) \\
&-&\log \sum_{i=1}^{n-1} \frac{(\sum_{j=1}
^n \lambda(x_j\mid\pa_j))^i}{i!}.\nonumber
\end{eqnarray}

The second term serves as a correction to ensure the score maintains the properties of an IT score, provided that 
the conditional distributions have densities with respect to the Lebesgue measure. 
\subsection{Monotonicity of scores }

\cite{Okati2024} formalize an intuitive principle of anomaly propagation: unless the connecting mechanisms also behave anomalously, a moderate outlier (measured by $\lambda(x_1)$) should not \emph{cause} an extreme outlier (measured by $\lambda(x_2)$). This can be stated for the bivariate causal relation $X_1\to X_2$ as follows:
\begin{lemma}[Weak IT outliers rarely cause strong ones]
\label{lm:con}
If $\lambda$ are IT scores and $X_1$ is a continuous variable, 
the following inequality holds:
\[
P( \lambda(X_2)\geq \lambda(x_2)\mid \lambda(X_1)\geq c ) \leq 2^{c-\lambda(x_2)}.
\]
\end{lemma}

In other words, whenever we generate an anomaly $x_1$ randomly from $P(X_1\mid \lambda(X_1)\geq c)$
and then generate $X_2$ via its mechanism $P(X_2\mid X_1)$, the resulting outlier score is unlikely to be much larger than $c$. The reason is that the probability of an event of the form $\lambda(X_j)\ge c$ decreases exponentially in $c$.

Note that \Cref{lm:con} says nothing about any individual $x_1$ -- some values of $x_1$ may generate stronger outliers with high probability, although ``most'' do not. For instance, it's easy to construct a deterministic relation $X_2=f(X_1)$, whose non-linearity maps a tiny region of non-anomalous values $x_1$ to anomalous values $x_2$. On the other hand, whenever this occurs, it must either be the case that the nonlinearity of $f$ is specially ``tuned'' to pick out those specific $x_1$; or else, $x_1$ is anomalous in some other sense not captured by the feature $\tau_1$. The algorithmic Markov condition will rule out the former case, prompting the need for a general outlier score to capture the latter case.

\subsection{Limitations of current approaches}

\paragraph{Rigid definition of outliers}
$\lambda_\tau$ defines outliers in terms of a feature statistic $\tau$ that is chosen in advance. Sometimes, we only know what makes a sample abnormal after seeing it, so we would like to be able to choose $\tau$ with hindsight. For example, suppose $P$ is univariate Gaussian. It seems natural to define $\tau(x) := |x-\mu|$, so that the anomaly detector flags extremely high or low values. However, we might also like to flag observations such as $x=0$ or $x=\mu$, as it would be a surprising coincidence to see precisely these values. While the probability \emph{density} at these points may be high, intuitively it seems much more likely to guess these specific numbers as a result of some alternative process different from $P$, than it is to sample them from $P$ (which indeed occurs with probability zero).

As another example, suppose each component of a continuous multivariate observation $x$ represents the reading of a different sensor. The signal transmission from all $d$ sensors may break in such a way that all components give the same constant ``idle'' state $c$. Such coincidences also indicate an unusual event.

\cite{Aggarwal2016}
describe a broad variety of different outliers beyond the above toy examples: These 
can be unusual frequencies of words in text documents \citep{mohotti2020efficient}, unexpected patterns in images \citep{yakovlev2021abstraction},
unusually large cliques
in graphs \citep{hooi2016fraudar}, or points lying 
in low-density regions \citep{breunig2000lof}. 

\paragraph{No general decomposition rule}

While \eqref{eq:conv} nicely decomposes the joint outlier score into mechanism-specific conditional scores, 
a key limitation is that this decomposition relies on {\it defining} the joint score based on a sum of conditional scores.  
This does not suggest that any reasonable outlier score (e.g., using a generic feature function $\tau$) for the joint observation can be decomposed in this manner.

\subsection{Our contributions}

This paper's contributions are purely conceptual. We do not propose a practical algorithm for outlier detection or root cause analysis, as our core definitions can only be computed in the infinite-runtime limit. Instead, we provide a theoretical framework for calibrating and interpreting outlier scores. The framework ensures that outlier scores meet three critical criteria:
(i) Comparability across diverse probability spaces and data modalities. (ii) Non-increasingness along causal chains of downstream effects, regardless of variable modalities. (iii) Well-defined attribution of joint system outlier scores to anomalies of mechanisms.

Our ideas were guided by the following general working hypothesis, which we believe applies far beyond the subject of this paper: 
\begin{principle}[Information Theory as a Guide]\label{pr:it}
Good information-theoretic concepts enable many nice theorems, but they are often hard to work with in practice. 
However, together with distributional assumptions 
(e.g. Gaussianity) they can boil down to simple concepts (e.g. linear algebraic expressions). The resulting formulae may be valid and useful beyond the distributional assumptions (e.g. by virtue of linear algebra).  
\end{principle}
While Shannon information is sometimes hard to
estimate from small sample sizes, {\it algorithmic} information is even worse: Kolmogorov complexity is not even computable. Furthermore, its identities hold only up to machine-dependent additive constants. Therefore, we owe it to the reader to show that our algorithmic information theoretic concepts trigger insights that can be applied in practice, as we will try in \cref{sec:ex}. 

The paper is structured as follows. \cref{sec:key}
reviews concepts from statistics, information theory, and causality that we build upon. \cref{sec:dec}
derives the decomposition, and uses it to show that weak outliers do not cause stronger ones. \cref{sec:ex} discusses simple examples, and \cref{sec:toy} a toy experiment. For a cleaner exposition, we defer some formal proofs and definitions to the supplementary material.

\section{Key ingredients}
\label{sec:key}

\subsection{Statistical testing with e-values instead of p-values}

While p-values are the most famous measure of evidence in statistical testing, e-values are recently gaining popularity for their superior ability to aggregate evidence across multiple tests \citep{ramdas2024hypothesis}. These values are inconsistently scaled in the literature, with e-values being comparable to reciprocals of p-values and exponentials of algorithmic randomness scores. To remove any obfuscation arising from scaling conventions, we introduce the following definitions.

\begin{definition}
\label{def:tests}
A probability-bounded test (\textbf{p-test}) in ratio form is a statistic $\Lambda:\domain\rightarrow[0,\infty]$  satisfying $\forall \eps>0,
P\left( \Lambda(X)
\ge 1/\eps \right)
\le \eps.$
An expectation-bounded test (\textbf{e-test}) in ratio form is a statistic $\Lambda:\domain\rightarrow[0,\infty]$ satisfying $\E_{X\sim P}\left( \Lambda(X) \right)
\le 1.$
\end{definition}

We say a statistic $\Lambda:\domain\rightarrow[0,\infty]$ is a p-test (or e-test) in \emph{probability form}, if $1/\Lambda$ is a p-test (or e-test) in ratio form. Note that what is commonly called a ``p-value'' is a p-test in probability form, satisfying $\Lambda(X) \leq \epsilon$ with probability at most $\epsilon$. In contrast, what the literature calls an ``e-value'' is an e-test in \emph{ratio} form.

Similarly, a statistic $\lambda:\domain\rightarrow[-\infty,\infty]$ is a p-test (or e-test) in \emph{log form}, if $2^\lambda$ is a p-test (or e-test) in ratio form. Our convention is to use lowercase symbols like $\lambda$ to indicate log form. For more details on tests, see \cref{sec:etestsandptests} in the supplementary material.

p-tests provide a straightforward way to control the Type I error rate, i.e., false positives. While we want to be flexible about what kinds of anomalies to consider, we should insist that \emph{most} samples from $P$ are not outliers. To ensure that samples from $P$ are labeled as outliers at a rate no larger than a desired threshold $\epsilon$, we can choose a p-test and flag only those samples whose scores are above $1/\epsilon$ (i.e., below $\epsilon$ when expressed in probability form).

By Markov's inequality, every e-test is also a p-test. Hence, e-tests achieve the same false positive rate; however, they are more conservative. In return, e-tests offer many advantages related to composability and optional stopping \citep{grunwald2020safe,ramdas2023game,ramdas2024hypothesis}. Our main result, \Cref{thm:dec}, concerns a general-purpose e-test that conveniently decomposes into a sum of contributions, attributable to a causal network's mechanisms. 

\subsection{Basic notions from algorithmic information theory} 

In a very general sense, an observation $x$ may be deemed an outlier whenever it belongs to a low-probability set that is flagged by some computable test. This condition can be rephrased in terms of the length of a compressed description of $x$.

To formalize description lengths, we fix a universal prefix-free Turing machine, and discretize our variables so that their values can be written as strings. The conditional \emph{Kolmogorov complexity} $K(x\mid y)$ is the bit length of the shortest program \( p \) that outputs \( x \) when given access to another string $y$  \citep{Vitanyi19}. It satisfies the Kraft inequality \( \sum_x 2^{-K(x \mid y)} < 1 \). When $y$ is the empty string, we simply write $K(x)$. Just as Shannon's entropy measures information content for a probability distribution, $K(x)$ measures it for an individual sample $x$.

Using a standard prefix-free encoding of \( n \)-tuples, \citet{Vitanyi19} also define the joint Kolmogorov complexity \( K(x_1, \dots, x_n) \). By analogy with Shannon's mutual information, they define the \emph{algorithmic mutual information between} $x$ and $y$, conditional on $z^*$, by
\begin{align*}
I(x:y\mid z^*) 
&:= K(x\mid z^*) + K(y\mid z^*) - K(x,y\mid z^*) \\
&\ceq K(x\mid z^*) - K(x\mid (y,z)^*).
\end{align*}
Here, $z^*$ denotes a shortest program that outputs $z$, and $\ceq$ denotes equality up to a constant dependent on the universal machine, but not on $x$ or $y$\footnote{Note that $x^*$ contains more information than $x$, because $x$ is easily generated from $x^*$ but not vice versa. Conditioning on $x$ instead of $x^*$ would result in error terms that are often constant, and at worst logarithmic in the length of $x$ \citep{gacs2021lecture}.}. We say $x$ and $y$ are \emph{conditionally independent} given $z^*$, abbreviated as $x\independent y\mid z^*$, if $I(x:y\mid z^*)\ceq 0$.

\subsection{Universal tests}

Semicomputable p- and e-tests are called Martin-Löf and Levin tests, respectively. Each class of tests can be aggregated, resulting in universal tests that effectively combine every conceivable anomaly scoring algorithm, in order to detect the broadest possible variety of outliers. \cref{sec:constructuniversal} in the supplementary material contains more details; here we give a brief overview.

\begin{definition}[Domination property]
For two statistical tests $\lambda_1$ and $\lambda_2$ expressed in log form, we say $\lambda_1$ \emph{dominates} $\lambda_2$ if there exists a constant $c \in \mathbb{R}$ such that for all observations $x$ in the sample space,
\begin{align*}
    \lambda_1(x) \geq \lambda_2(x) - c.
\end{align*}
Intuitively, this means $\lambda_1$ can detect at least all the anomalies that $\lambda_2$ can detect, up to a constant term.
\end{definition}

The domination property provides a natural way to compare the power of different statistical tests. Since \Cref{def:tests} restricts us from increasing a test everywhere, it is perhaps surprising that the class of semicomputable tests contains a \emph{universal} test that dominates all the others.

\begin{theorem}[Universality of randomness deficiency]
Let $\cX$ be a discrete space that can be interpreted as a subset of $\{0,1\}^*$ in a canonical way, and $P$
be a computable probability distribution on $\cX$ (i.e., with finite description length). 
Then, the randomness deficiency of $x\in\cX$, defined by
\begin{equation}
\label{eq:universalfirst}
\delta (x) := -\log P(x) - K(x\mid P^*),
\end{equation}
is a universal Levin test, dominating all other Levin tests (i.e., semicomputable e-tests).
\end{theorem}

The intuition behind \cref{eq:universalfirst} is that typical samples from a distribution $P$ are optimally compressed by encodings of length $-\log P(x)$. When a sample $x$ can be compressed beyond this theoretical limit, we consider it anomalous.
This includes the case where we observe $x=0$ from a discretized centered Gaussian distribution. Despite being the mode of the distribution, $0$ is considered anomalous because its negative log-likelihood $-\log P(x)$ is substantially larger than its Kolmogorov complexity $K(0)\ceq 0$.

\begin{example}[Uniform distribution]
When $P$ is uniform over all $2^d$ binary strings of length $d$, \labelcref{eq:universalfirst} becomes $\delta(x) = d - K(x\mid P^*)\ceq d - K(x\mid d)$, measuring the degree to which $x$ is compressible.
\end{example}

When a distributional estimate $\hat{P}$ is inferred from data, \citep{Bishop1993} uses the ``reconstruction loss'' $-\log \hat{P}(x)$ as an outlier score. Without further considerations, this score is not calibrated, because there may be a huge region with low density, with high total probability. However, if the term $K(x\mid P^*)$ is also small, then we can say that $x$ is special within this huge set.\footnote{Note that conformal p-values \citep{Bates2021} do not require any parametric assumptions or density estimation, since their coverage guarantees rely on exchangeability alone. However, statements on {\it conditional} coverage \citep{Barber2019} are too weak for our purpose of calibrating {\it conditional} outlier scores.} 

\subsection{Algorithmic Markov condition} 

In order to identify root causes of an anomaly, we would like to decompose the randomness deficiency into a sum of contributions from each causal mechanism in a Bayesian network. Our framework for doing so requires an algorithmic analogue of the causal Markov condition. Specifically, \cite{Algorithmic} propose an adaptation of \eqref{eq:markov} that characterizes algorithmic dependencies between {\it individual objects}, rather than statistical dependencies between {\it random variables}:
\begin{postulate}[Algorithmic Markov condition]
\label{algmarkov}
Let $x_1,\dots,x_n$ be binary words describing objects whose causal relations are given by the DAG $G$. Then, the joint complexity of $(x_1,\dots,x_n)$ decomposes as
\[
K(x_1,\dots,x_n) \ceq \sum_{j=1}^n K(x_j\mid\pa^*_j).
\]
Equivalently, for any three sets $R,S,T$ of nodes such that $R$ d-separates $S$ and $T$, we have
\[S\independent T\mid R^*.\]
\end{postulate}

\cite{Algorithmic} argue that,
within a causal Bayesian network whose nodes represent random variables, 
the mechanisms $P_{X_j\mid \PA_j}$ 
constitute independent information-bearing objects.
Given $m$ observations of $n$-tuples $(x^1_1,\dots,x_n^1), \dots, (x^m_1,\dots,x_n^m)$, they 
define a DAG $G^m$ on $n\times (m+1)$ nodes, corresponding to the $n$ mechanisms $P_{X_j\mid \PA_j}$ and $n\times m$ observations $x_j^l$. The edges of $G^m$ consist of assigning the observations $\pa^l_j$ and the \emph{node} $P_{X_j\mid \PA_j}$ as parents of $x_j^l$. \cref{fig:x--y} shows $G^m$, when $G$ is a cause-effect pair $X\to Y$. 

For our purposes, it suffices to consider $G^1$, containing only one observation $x_j$ for each node $X_j$ in $G$. Accordingly, \( G^1 \) is constructed from \( G \) as follows: replace each variable \( X_j \) with its observation \( x_j \), and add \( P_{X_j \mid  \PA_j} \) as a parent of \( X_j \). Each \( P_{X_j \mid  \PA_j} \) becomes a root node, and these are the only root nodes in \( G^1 \).

When the mechanisms \( P_{X_j \mid  \PA_j} \) appear as root nodes, \Cref{algmarkov} implies their algorithmic independence from one another; \cite{Algorithmic} refer to this as \emph{Independence of Mechanisms}. More generally, we can allow mechanisms to share information by adding a new root node as their shared parent. In either case, the algorithmic Markov condition can be interpreted as saying that multiple mechanisms cannot fail in a correlated manner. Instead, correlated anomalies must be attributable to a common root cause. To derive a formal statement, we apply \Cref{algmarkov} to $G^1$:
\begin{lemma}[Conditional irrelevance of other mechanisms and predecessors when parents are given]
Let $X_1,\dots,X_n$ be causally ordered. 
Given its parents $\pa_j$ and the mechanism
$P_{X_j\mid\PA_j}$, none of the 
other mechanisms $(P_{X_i\mid\PA_i})_{i\neq j}$ 
and none of the causal 
predecessors $(x_i)_{i<j}$ enable further compression of $x_j$. That is,
\label{lm:indepinputs}
\begin{align*}
x_j \independent 
(x_i)_{i<j}, (P_{X_i\mid\PA_i})_{i\ne j}\mid(\pa_j,
P_{X_j\mid\PA_j})^*.
\end{align*}
\end{lemma}

\section{Decomposition of Randomness Deficiency}
\label{sec:dec}

We start by presenting the decomposition of randomness deficiency in the bivariate case. 
Specifically, consider the DAG $X \to Y$ between a cause $X$ and its effect $Y$. Extending \cref{eq:universalfirst}
to joint and conditional distributions, define the joint randomness deficiency of outcomes $(x,y)$ with respect to $P_{X,Y}$ by
\begin{align}
\label{eq:deltaxy}
    \delta(x, y)
    \defeq
    - \log P(x, y) - K(x, y \mid (P_{X, Y})^*),
\end{align}
and the conditional randomness deficiency by
\begin{align}
    \label{eq:deltaymidx}
    \delta (y \mid x)
    \defeq 
    -\log P (y \mid x) - K(y \mid (x, P_{Y \mid X})^*).
\end{align}
\begin{restatable}[Decomposition of randomness deficiency for a cause-effect pair]{lemma}{decompositionpair}
\label{lm:deltaxy}
For any two random variables $X \to Y$ (i.e., $X$ being the cause of $Y$), and for individual observations $x$ and $y$,  
the following equality holds under \cref{algmarkov} for the graph $G^1$ in \cref{fig:x--y}:
\begin{align*}
    \delta (x, y) \ceq \delta (x) + \delta (y \mid x) 
\end{align*}
\end{restatable}

\noindent The proof relies on \cref{lm:indepinputs}, and is provided in \cref{sec:decomp} of the supplementary. The following example shows why the algorithmic Markov condition is essential in order for the randomness deficiency to be additive.

\begin{figure}
    \centering
\begin{tikzpicture}[
    node distance=0.4cm and 0.5cm,
    every node/.style={draw, circle, minimum size=0.5cm},
    every edge/.style={->, thick}
]
\node (x2) [inner sep=4.5pt] {$x^2$};
\node (S) [left=of x2] {$P_X$};
\node (x1) [above=of x2, inner sep=4.5pt] {$x^1$};
\node (xm) [below=of x2] {$x^m$};

\node (y1) [right=of x1, inner sep=4.5pt] {$y^1$};
\node (y2) [right=of x2, inner sep=4.5pt] {$y^2$};
\node (ym) [right=of xm] {$y^m$};

\node (M) [right=of y2, minimum size=0.4cm, inner sep=1pt] {\small $P_{Y\mid X}$};

\draw[->] (S) -- (x1);
\draw[->] (S) -- (x2);
\draw[->] (S) -- (xm);
\draw[draw=none] (x2) --node[midway, draw=none, fill=none] {$\dots$} (xm);

\draw[->] (x1) -- (y1);
\draw[->] (x2) -- (y2);
\draw[->] (xm) -- (ym);

\draw[->] (M) -- (y1);
\draw[->] (M) -- (y2);
\draw[draw=none] (y2) -- node[midway, draw=none, fill=none] {$\dots$} (ym);
\draw[->] (M) -- (ym);
\end{tikzpicture}
\caption{The graph $G^m$, derived from $G$ of the form $X\to Y$, for $x^1, \dots, x^m$ sampled from $P_X$ and $y^1, \dots, y^m$ sampled from  $P_{Y \mid X}$.}
\label{fig:x--y}
\end{figure}

\begin{example}[No additivity without the algorithmic Markov condition]\label{ex:icm}
Let $\cX=\cY=\{0,1\}^d$
and $P_X$ be the uniform distribution. Furthermore, let $P_{Y\mid X}$ be the deterministic mechanism
\[
P(y\mid x) = 1 \quad \hbox{ iff } \quad 
y= x\oplus x_0,
\]
where $\oplus$ denotes bitwise XOR and $x_0$
is an algorithmically random string with $K(x_0) \ceq d$. Since $P_{Y\mid X}$ is deterministic, we always have $\delta(y\mid x)\ceq 0$.

Consider the input $x=x_0$, which violates the conditional independence relation $x \independent P_{Y\mid X} \mid (P_X)^*$. 
Then, $\delta(x)\ceq 0$ because $x$ is random with respect to $P_X$, but $\delta(x,y)\ceq d\gg 0$ because $(x,y)=(x_0,0^d)$ is a simple function of $x_0$, which is easily deciphered from $P_{X,Y}$. Since $x$ is non-generic relative to $P_{Y\mid X}$, we find that the randomness deficiency of $(x,y)$ cannot be attributed to $x$, nor to the mechanism generating $y$ from $x$. 

\end{example}

The multivariate generalization of \cref{lm:deltaxy} is this paper's main result:

\begin{restatable}[Decomposition of joint randomness deficiency]{theorem}{decomposition}
\label{thm:dec}
Let the set of strings $x_1, x_2, \dots, x_n$ be causally connected by a directed acyclic graph $G$, so that the algorithmic Markov condition holds for $G^1$. Then, the joint randomness deficiency of $x_1, x_2, \ldots, x_n$ decomposes into a sum of conditional randomness deficiencies of the mechanisms:
\[
    \delta(x_1,\dots,x_n) \ceq \sum_{j=1}^n \delta (x_j\mid\pa_j),
\]
where $\delta (x_j\mid\pa_j)$ denotes the randomness deficiency of $x_j$ given its parents.
\end{restatable}
The proof is by induction over the number $n$ of nodes; in order to chain the $\ceq$ equalities, we treat $n$ as a constant. See \cref{sec:decomp} in the supplementary for details, including a proof of the following corollary:

\begin{restatable}[Weak outliers do not cause stronger ones]{corollary}{monotonicity}
\label{thm:mon}
If there is a unique root cause
$j \in \{1,\dots,n\}$,
in the sense that 
\[
\delta(x_i\mid \pa_i) \ceq 0 \quad \hbox{ for } i\neq j,
\]
and the conditions of \cref{thm:dec} are met, then
\begin{equation}\label{eq:mon}
\delta(x_i)\cle  \delta (x_j\mid \pa_j) \quad \forall i\in \{1,\dots,n\}.
\end{equation}
\end{restatable}
On the bivariate graph $X\to Y$ whose unique root anomaly is $X$, \Cref{thm:mon} implies $\delta(y)\cle\delta(x)$. More generally, \eqref{eq:mon} states that no node has a randomness deficiency exceeding that of the root cause. 
In a nutshell, ``weak outliers cannot cause strong ones''. --
We have thus found an AIT version of \cref{lm:con}.\footnote{Hence, in the language of modern resource theories in physics (\citep{Coecke2019}, Definition 5.1), an anomaly is a ``resource'', and $\delta$ is a ``monotone'' measuring its value.}

It is worth noting that \citet{levin1984randomness} (see also \citet{gacs2021lecture}) demonstrated a similar principle, calling it \emph{randomness conservation}. It says that a mechanism's typical output cannot exhibit a substantially larger randomness deficiency than its input - under the condition that the mechanism is simple, i.e., has constant description length. In physics, randomness deficiency corresponds to a lack of entropy \citep{zurek1989algorithmic,gacs1994boltzmann}. Therefore, the second law of thermodynamics amounts to an instance of randomness conservation \citep{ebtekar2025foundations}. By conditioning the randomness deficiencies on the mechanisms, we obtained results that hold even when the mechanisms are not simple.

\section{Relation to computable anomaly scores}
\label{sec:ex}
In practice, we cannot directly use the randomness deficiency to detect outliers, because it is not computable. However, it is \emph{semicomputable}: by searching over descriptions of $x$, one can gradually improve the upper bound on its Kolmogorov complexity, thus obtaining approximations that monotonically approach the randomness deficiency in the limit. Since a generic search would be extremely slow, the search should focus on features or models suitable to the domain in question. We leave the development of good domain-specific anomaly search algorithms to future work.

Alternatively, the search can be performed manually without a computer algorithm. During statistical analysis, one can identify any computable feature, perhaps even \emph{after} observing $x$, and from it derive a lower bound on the randomness deficiency.

This section describes simple scenarios in which such bounds boil down to simple and well-known outlier scores. Thus, the randomness deficiency provides a practical recipe for obtaining and calibrating computable scores. Moreover, its theoretical properties motivate us to seek computable scores with the same desirable properties: {\it decomposition of joint scores into scores of the mechanisms}, and {\it monotonicity under marginalization}.
\begin{example}[z-score]
\label{ex:zscore}
For a Gaussian variable $X\sim {\mathcal N}(\mu,\sigma)$, the squared z-score reads
$
z^2(x) := (x-\mu)^2/\sigma^2.
$
At a fixed level of precision, $K(x\mid P^*)\cle 
2\log |x-\mu|$. 
Substituting the log likelihood
\begin{align*}
-\log P(x) &= \log\sqrt{2\pi}\sigma + \frac{\log e}{2} z^2(x),
\end{align*}
and treating $\sigma$ as a constant, yields
\begin{align*}
\delta (x) &\cge 
\frac{\log e}{2} z^2(x)  -
2 \log |x-\mu|.
\end{align*}
\end{example}

Note that while the well-known identity $\Exp[z^2]=1$ tells us that $z^2$ is an e-value, the universal e-value $2^{\delta}\approx e^{z^2/2}$ provides far stronger evidence at the tails.
Now we generalize to $n > 1$ dimensions:
\begin{example}[Decomposing squared Mahalanobis distance]\label{ex:maha}
The randomness deficiency of a random vector $\bx \in \R^n$, drawn from a centered Gaussian, satisfies
\[
\delta(\bx)
\cge \frac{\log e}{2} \bx^T \Sigma^{-1}_\bX \bx
- O(\log \|\bx\|_\infty),
\]
where $\Sigma_\bX$ denotes the covariance matrix
and $n$ is treated as a constant. 
The leading term is $(\log e) / 2$ times the squared Mahalanobis distance (M-distance, for short).

Let $\bX$ be generated by  
a causal Bayesian network with
linear structural equations  
$\bX = A \bX + {\bf N}$, where $A$ is strictly lower triangular and 
 $N_i$ are independent noise variables with variance $\sigma_i^2$. 
Using the transformation 
${\bf N} = (I-A) \bX$, we obtain a diagonal form
\[
\bx^T \Sigma^{-1}_\bX \bx = \sum_{i=1}^n \frac{n_i^2}{\sigma_i^2}.
\]
For large $n_i$, the expression $n_i^2/\sigma_2$
is roughly proportional to the randomness deficiency of the mechanism $P(X_i\mid PA_i)$. Hence, 
decomposition of randomness deficiencies translates asymptotically to decompositions of the squared M-distance (as used for multivariate anomaly detection \cite{Aggarwal2016}) into $z^2$-scores. 
Furthermore, M-distance 
is non-increasing with respect to 
marginalization to a subset of variables (see \cref{sec:mahalanobis} in the supplementary), resembling the monotonicity of 
randomness deficiency. 
\end{example}
\cref{ex:maha} supports \cref{pr:it}: a conservative bound on the randomness deficiency yields a computable e-value, as a function of the $z^2$-scores. Moreover, we verified using linear algebra that the $z^2$-score of any variable cannot exceed the sum of all the ``noise scores''. 
 
Note that the upper bound in \Cref{thm:mon} is a {\it conditional} randomness deficiency. The following example shows that a root cause may have a smaller {\it marginal} randomness deficiency than its downstream effects.
\begin{example}[Root cause with small marginal score]
With $N_i\sim {\mathcal N}(0,1)$, consider the three-node model
\begin{align*}
X_1 &:= N_1,
\\X_2 &:= 2 X_1 + N_2,
\\X_3 &:= X_1 - X_2 + N_3. 
\end{align*}
Then, $X_2\sim {\mathcal N}
(0,\sqrt{5})$.
Suppose $n_1,n_3$ take on typical values, while $n_2$ is anomalously large. Then, the marginal randomness deficiency $\delta(x_2)$ increases with $n_2^2/5$,
while the conditional randomness deficiency $\delta(x_2\mid x_1)$ increases with $n^2_2/1$ (where we've ignored the factor of $\frac{\log e}{2}$ from \Cref{ex:zscore}).
Since $X_3 \sim 
{\mathcal N}(0,\sqrt{3})$,
$\delta(x_3)$ increases with $n^2_2/3$. 
Thus, for large perturbations of $x_2$, we have $\delta(x_2)\ll\delta(x_3)\ll\delta(x_2\mid x_1)\approx\delta(x_1,x_2)$. Looking only at marginal scores, it would seem as though $x_2$ caused a stronger outlier $x_3$. Nonetheless, $\delta(x_3)\cle\delta(x_2\mid x_1)$ in agreement with \Cref{thm:mon}.
\end{example}

This paradox resembles Example 2.1 in \cite{Li2024}, where it is phrased in terms of $z^2$-scores. 
Following \cref{pr:it}, we can describe it entirely in terms of M-distances (which are just $z^2$-scores in the one-dimensional case):
the M-distance  of $x_3$ 
can be larger than the M-distance of its root cause $x_2$, but not larger than the M-distance of the pair $(x_1,x_2)$, i.e., the full cause of $x_3$. In other words, to fully quantify  
the anomaly introduced by perturbing $x_2$, we must also account for the destroyed coupling between $x_1$ and $x_2$, not only for the size of the value $x_2$ itself.

The insight from algorithmic information theory is that the anomaly scores of the vector $(x_1,x_2)$ and
the scalar $x_3$ are indeed comparable, and also comparable to the conditional score of $x_2$ given $x_1$, because their calibration was guided by the randomness deficiency. 

In the following example, which could be easily generalized to words over an arbitrary alphabet, 
randomness deficiency essentially boils down to relative entropy:
\begin{example}[$m$-bit binary word]
\label{ex:bin}
Let us first consider an $m$-bit word  whose bits are set to $1$ independently with probability $p$.
Counting the number of words $w$
with fixed Hamming weight $\ell$ yields the lower bound
\[
\delta(w) \cge - \ell \log p -
(m-\ell) \log (1-p) 
- \log {m \choose \ell}
- O(\log m).
\]
Using Stirling's approximation $\log m! = m \log m - m \log e + O(\log m)$,
one can show that
\begin{equation}\label{eq:kl}
\delta(w)
\cge m\cdot \kldiv{\frac \ell m}{p} - O(\log m), 
\end{equation} 
with the binary Kullback-Leibler distance
\[
\kldiv{q}{p} :=  
q\log\frac qp + (1-q) \log \frac{1-q}{1-p}.
\]
Here we see that words with
unexpectedly low or unexpectedly Hamming weight are both assigned high outlier scores, enabling us to compare both types of outliers. This works despite the probability mass function being monotonic in the Hamming weight $\ell$.\footnote{\cite{thomas2006elements}, in Section 11.2, call a word ``$\epsilon$-typical'' when it satisfies 
$\kldiv{\frac \ell m}{p} \leq \epsilon$. Their equation (11.67) finds that
such a word occurs with probability at least
$1- (m+1)^2 \cdot e^{-\epsilon m} $.} 
\end{example}
Many works \citep{Aggarwal2016,Akoglu2012,Noble2003} propose detecting anomalies by compression, but consider anomalies that have {\it higher} compression length than usual, seemingly in conflict with the randomness deficiency which flags {\it low} compression lengths. \cref{ex:bin} resolves this paradox:
for $p < 1/2$, any word with $K(w) > m \cdot H(p)$
must have Hamming weight larger than $m p$. Accordingly, its randomness deficiency is bounded from below by \eqref{eq:kl}. Hence, paradoxically, an unusually {\it high} compression length
also implies non-zero randomness deficiency, because it entails 
an increase of the log likelihood term that outweighs 
the observed increase in compression length.

More sophisticated notions of anomalies can be obtained, for instance, when $X$ is graph-valued and an anomaly is given by large cliques \citep{Aggarwal2016} -- e.g. when fraudsters work together frequently on illegal activities, their communication
 is densely connected. 
To estimate the
random deficiency of a graph $G$ with a clique of size $k$, it suffices to
define a probability distribution on the set of graphs with $m$ nodes, and bound a graph's description length by counting the number of graphs with cliques of size $k$.

\section{Experiment with Lempel-Ziv Compression}
\label{sec:toy}

So far we have accounted for the term $K(x\mid P^*)$ only by very rough upper bounds rather than trying to approximate it.
Here
we describe a toy scenario in which we can detect anomalies using the Lempel-Ziv compression algorithm; more difficult scenarios may demand more powerful compression algorithms that are closer to general AI.

Consider the causal DAG:
$
X_1 \to X_2 \to \cdots \to X_n,
$ 
with the structural equations
\begin{equation}
\label{eq:sem}
X_1 = N_1; \quad X_j = X_{j-1} + N_j \quad \text{for } j = 2,\dots,n.
\end{equation}
Moreover, suppose that every $N_j$ for $j=1,\dots,n$
is drawn from the uniform distribution on the set of numbers in $[0,1]$ discretized to $d\gg 1$ digits of precision. 
By choosing the uniform distribution, the terms $-\log P(x_j\mid x_{j-1})$ become constant; moreover, we have $K(x_j\mid (x_{j-1}, P)^*)\ceq K(x_j\mid x_{j-1})$.
Since the conditional distribution of each $X_j$, given its parent, is uniform over the (discretized) interval $[x_{j-1},\, x_{j-1}+1]$, the conditional randomness deficiency of 
$x_j$ reads
\[
\delta(x_j\mid x_{j-1}) \ceq d\cdot \log 10 - K(x_j\mid x_{j-1}).
\]

Suppose now we inject an anomaly at some node $j$ by
setting $n_j$ to some numbers in $\{0,\dots,0.9\}$. 
As a result, $x_j$ and $x_{j-1}$ now coincide in at most $2$ digits,
so that $K(x_j\mid x_{j-1}) \ceq 0$ and $\delta(x_j\mid x_{j-1}) \ceq d\cdot \log 10$. 

Following \cite{submodular}, we approximate 
$K(x_j\mid x_{j-1}^*) \ceq K(x_j, x_{j-1}) - K(x_{j-1})$ by 
$
R(x_j,x_{j-1}) - R(x_{j-1}),
$
where $R$ denotes the length of a compressed encoding using the Lempel-Ziv algorithm.
Whenever $n_j$ corrupts to a $1$-digit number, 
Lempel-Ziv recognizes that $x_j$ and $x_{j-1}$
coincide with respect to all but $1$ or $2$ digits. Thus,
$R(x_j,x_{j-1})\approx R(x_{j-1})$, which lets us infer the randomness deficiency of almost $d\log 10$. 

We conducted experiments to verify our findings for
$n=4$, with the noise uniformly drawn from
numbers in $[0,1]$ with $d=10$ digits.  
We randomly chose one of the $4$ nodes as ``root cause'' and set all but one digit of its noise variables to zero. To detect the root cause, we selected the label $j$ that minimized Lempel-Ziv compression length
and found the right one in $100$ out of $100$ runs. 

From a theoretical point of view, the example shows that the joint observation can be anomalous (here in the sense of showing two variables whose digits coincide largely), although none of the variables 
show unexpected behaviour with respect to their \emph{marginal} distribution. 
Hence, the root cause can only be found by inspecting which \emph{mechanism} behaves unexpectedly.

\section{Conclusions}
The algorithmic randomness deficiency offers a principled unified definition of outliers, without prior specification of the feature that exposes the anomaly. 
On a causal Bayesian network, we saw that the randomness deficiency of a joint observation decomposes along individual causal mechanisms, subject to the Independent Mechanisms Principle. This allows us to trace anomalous observations back to their root causes, identifying specific mechanisms responsible for the anomaly. Moreover, we showed that weak outliers cannot cause stronger outliers,
extending Levin’s law of randomness conservation.
These foundational insights can help
calibrate anomaly scores as well as inspire more versatile anomaly detection algorithms, supporting root cause analysis in complex systems. 

{\bf Acknowledgments:} DJ and YW would like to thank the authors of \cite{Li2024} for an interesting discussion about their Example 2.1.


\newpage

\onecolumn

\title{Toward Universal Laws of Outlier Propagation\\(Supplementary Material)}
\maketitle

{
    \def\thefootnote{*}\footnotetext{These authors contributed equally to this work.}
}

\section{p-tests and e-tests}
\label{sec:etestsandptests}

\cref{tab:testssummary} summarizes key differences between p-tests and e-tests. While p-tests are linked to controlling false positive rates and p-values, e-tests are closely associated with likelihood ratios and betting game interpretations.

\begin{table}[ht]
\centering
\setlength{\tabcolsep}{5pt} 
\renewcommand{\arraystretch}{1.2} 
\begin{tabular}{c|c|c}
\hline
      & p-tests (Martin-L\"of) & e-tests (Levin) \\ 
\hline
Intuition &
False positives / p-values &
Betting games / likelihood ratios \\
Defining property &
$\forall\epsilon>0,\,P\left( \Lambda(X) \ge 1/\epsilon \right) \le\epsilon$ &
$\E_{X\sim P}\left( \Lambda(X) \right) \le 1$ \\
Prototypical example &
$1/P(\tau(X) \ge \tau(x))$ &
$Q(x)/P(x)$ \\
Combination &
$\sup_i w_i\Lambda_i$ &
$\sum_i w_i\Lambda_i$ \\
Combination (log form) &
$\sup_i \{ \lambda_i + \log w_i \}$ &
$\log\left(\sum_i w_i 2^{\lambda_i}\right)$ \\
Universal test (log form) &
$\sup_i\{ \lambda_i(x) - K(i\mid P^*) \}$ &
$-\log P(x) - K(x\mid P^*)$ \\
\hline
\end{tabular}
\caption{Summary of types of test statistics (outlier scores)}
\label{tab:testssummary}
\end{table}

\paragraph{Converting between p-Tests from e-Tests}
By Markov's inequality, every e-test is also a p-test. Conversely, \citet{vovk2021values} describe a number of ways to \emph{calibrate} any given p-test into an e-test. A natural choice is the Ramdas calibration, which turns the p-test $\Lambda$ into the e-test
\begin{equation}
\label{eq:ramdascalibration}
\Lambda'(x) := \frac
{\Lambda(x)-\ln\Lambda(x)-1}
{\ln^2\Lambda(x)}.
\end{equation}

\begin{example}[One-tailed p-value]
For any feature statistic $\tau:\domain\rightarrow\R$, the one-tailed p-value is given by
\begin{equation}
\label{eq:pvaluetest}
\Lambda_\tau(x) := P(\tau(X) \ge \tau(x)).
\end{equation}
It is easily verified to be a p-test in probability form. The outlier score \labelcref{eq:itscore} is the same test as the p-value \labelcref{eq:pvaluetest}, but expressed in log form.
\end{example}

\begin{example}[Likelihood ratio]
For any alternative hypothesis described by a sub-probability distribution\footnote{The algorithmic information theory literature refers to sub-probability distributions as \emph{semimeasures}. They generalize probability measures by allowing their sum to be less than $1$.} $Q$, the likelihood ratio is given by
\begin{equation}
\label{eq:likelihoodratiotest}
\Lambda_Q(x) := \frac{Q(x)}{P(x)}.
\end{equation}
It is easily verified to be an e-test in ratio form; in fact, all e-tests can be written this way. By the Neyman-Pearson lemma \citep{neyman1933ix}, when $Q$ is a probability distribution, $\Lambda_Q$ is the optimal test to distinguish between $P$ and $Q$.
\end{example}

For the purposes of outlier detection, $\tau$ may be a feature for which a high value should be considered anomalous, while $Q$ might be a distribution that we expect would result from some kind of anomaly. In practice, we can design many tests, each specializing in a different kind of anomaly. The following lemma allows us to merge many such tests into one.
\begin{restatable}[Combination tests]{lemma}{combinationtests}
\label{lem:combinationtests}
Let $\Lambda_i$ (or $\lambda_i$) be a finite or countably infinite sequence of tests, with associated weights $w_i>0$ summing to at most $1$. Then:
\begin{itemize}[left=0pt]
   \item If $\Lambda_i$ are p-tests in prob. form, so is $\inf_i \frac{\Lambda_i}{w_i}$.
    \item If $\Lambda_i$ are p-tests in ratio form, so is $\sup_i w_i\Lambda_i$.
    \item If $\lambda_i$ are p-tests in log form, so is $\sup_i \{\lambda_i + \log w_i\}$.
    \item If $\Lambda_i$ are e-tests in prob. form, so is $\left(\sum_i\frac{w_i}{\Lambda_i}\right)^{-1}$.
    \item If $\Lambda_i$ are e-tests in ratio form, so is $\sum_i w_i\Lambda_i$.
   \item If $\lambda_i$ are e-tests in log form, so is $\log\left(\sum_i w_i 2^{\lambda_i}\right)$.
\end{itemize}
\end{restatable}

\begin{proof}
\cite{genovese2006false} state a similar result for finitely many independent p-tests, calling the combination test a \emph{weighted Bonferroni procedure}. \cite{vovk2021values} state a similar result for equally weighted e-tests. It is fairly straightforward to extend these works to our setting; for completeness, the proof is as follows. First, given some p-tests $\Lambda_i$ in ratio form, we verify \Cref{def:tests} for their combination $\sup_i w_i\Lambda_i$:
\begin{equation*}
P\left(\sup_i w_i\Lambda_i(X) \ge \frac 1\epsilon\right)
\le \sum_i P\left(\Lambda_i(X) \ge \frac 1{w_i\epsilon}\right)
\le \sum_i w_i\epsilon
\le \epsilon.
\end{equation*}
Next, given e-tests $\Lambda_i$ in ratio form, we verify \Cref{def:tests} for their combination $\sum_i w_i\Lambda_i$:
\begin{equation*}
\E\left(\sum_i w_i\Lambda_i(X)\right)
= \sum_i w_i \E(\Lambda_i(X))
\le \sum_i w_i
\le 1.
\end{equation*}
Transforming the tests into probability and log form yields the remaining results.
\end{proof}

\begin{example}[Two-tailed p-value]
By combining the one-tailed p-values \labelcref{eq:pvaluetest} for the feature statistics $\tau$ and $-\tau$, each with weight $0.5$, we obtain the two-tailed p-value
\begin{align*}
\Lambda_{\pm\tau}(x)
&:= \min\left\{\Lambda_\tau(x)/0.5,\;\Lambda_{-\tau}(x)/0.5\right)\}
\\&= 2\min\left\{
\Pr(\tau(X) \ge \tau(x)),\;
\Pr(\tau(X) \le \tau(x)) \right)\}.
\end{align*}
\end{example}

From now on, we express tests in log form except where stated otherwise. By \Cref{lem:combinationtests}, a combination test exceeds each of its component tests, up to the additive regularization term $\log w_i$ which does not depend on the sample $x$. Thus, by committing to a combination test prior to observing $x$, we effectively postpone the search over the component tests until after observing $x$.

Unfortunately, the regularization term does depend on $i$, and becomes arbitrarily large as the number of tests becomes large or infinite. If we want the combination test to be competitive with its most promising component tests, it becomes important to choose the weights well \citep{wasserman2006weighted}. The problem of choosing $w_i$ is analogous to that of choosing Bayesian priors, and is philosophically challenged by formal impossibility results in the theory of inductive inference \citep{adam2019no,wolpert2023implications}. Fortunately, computability poses a useful constraint on the set of permissible tests as well as priors \citep{rathmanner2011philosophical}. It is suggestive that every computable sequence of weights $w_i$ can be turned into a computable binary code with lengths $\lceil-\log w_i\rceil$ \citep{thomas2006elements}. Thus, minimizing the regularization penalty amounts to finding the shortest computable encoding for $i$.

\section{Constructing the universal e-test}
\label{sec:constructuniversal}

To construct a universal e-test, we first need a few concepts from computability theory. We say a function \( f: \{0,1\}^* \rightarrow \mathbb{R} \) is lower (or upper) semicomputable if it can be computably approximated from below (or above, respectively). For example, the Kolmogorov complexity $K$ is upper, but not lower, semicomputable: by running all programs in parallel, we gradually find shorter programs that output \( x \). We say $f$ is computable if it is both lower and upper and semicomputable.

We say a p-test or e-test is semicomputable, if it is lower semicomputable when expressed as a function in either ratio or log form. A semicomputable p-test is called a Martin-L\"of test \citep{martin1966definition}, while a semicomputable e-test is called a Levin test \citep{levin1976uniform}. Intuitively, semicomputable tests detect more anomalies as computation time increases. By restricting attention to semicomputable tests, it becomes possible to create universal combinations of them.

We return to the problem of optimizing the weights in \Cref{lem:combinationtests}: we would like to dominate not only each of the individual component tests, but also each of the combination tests obtainable by some computable sequence of weights $w_i$. Note that the constraints on $w_i$ amount to specifying a discrete sub-probability distribution. Among all lower semicomputable discrete sub-probability distributions,
\begin{equation}
\label{eq:univeralprobability}
m_i := 2^{-K(i\mid P^*)}
\end{equation}
is \emph{universal} in the sense that for all alternatives $w$ and all $i$,
\[\log m_i \cge \log w_i - K(w\mid P^*).\]
This also holds when $P$ is replaced by any other piece of prior knowledge.

Now, applying \Cref{lem:combinationtests} with the universal weights \labelcref{eq:univeralprobability}, to any sequence of p-tests $\lambda_i$ in log form, yields their universal combination
\begin{equation*}
\lambda(x)
:= \sup_i\{ \lambda_i(x) - K(i\mid P^*) \}.
\end{equation*}
In the case where $\{\lambda_i\}_{i=1}^\infty$ is a computable enumeration of some Martin-L\"of tests, $\lambda$ is itself a Martin-L\"of test. It follows that if $\{\lambda_i\}_{i=1}^\infty$ enumerates \emph{all} Martin-L\"of tests, then $\lambda$ is universal among them. Every program $p$ that outputs $i$ with access to $P$, determines the feature $\lambda_p:=\lambda_i$. Moreover, it satisfies $|p|\ge K(i\mid P^*)$, with equality for the shortest such program. Therefore, we can rewrite the universal Martin-L\"of test as
\begin{equation}
\label{eq:universalmartinlof}
\rho(x) := \sup_p\{ \lambda_p(x) - |p|\}.
\end{equation}

With e-tests, the situation is even nicer because \Cref{eq:likelihoodratiotest} provides a one-to-one correspondence between e-tests $\Lambda_Q$ and sub-probability distributions $Q$. Moreover, if we assume $P$ to be computable, an e-test for it is semicomputable (i.e., is a Levin test) iff its corresponding $Q$ is lower semicomputable. Applying \Cref{lem:combinationtests} with the universal weights \labelcref{eq:univeralprobability}, to the likelihood ratios $\Lambda_{Q_i}$, yields their universal combination
\begin{align*}
\Lambda(x)
:= \sum_i m_i\Lambda_{Q_i}(x)
= \frac{\sum_i m_i Q_i(x)}{P(x)}.
\end{align*}

In the case where $\{Q_i\}_{i=1}^\infty$ enumerates all lower semicomputable sub-probability measures, Theorem 4.3.3 in \cite{Vitanyi19} implies $\sum_i m_i Q_i(x) {\stackrel{\times}{=}} m(x\mid P^*)$, where ${\stackrel{\times}{=}}$ indicate equality up to a multiplicative term. Switching to log form, we obtain the universal Levin test
\begin{equation}
\label{eq:universallevin}
\delta(x)
:= \log\frac{m(x\mid P^*)}{P(x)}
= \log\frac{1}{P(x)} - K(x\mid P^*).
\end{equation}

By the Kraft inequality, it is an e-test in log form. Note that \labelcref{eq:universallevin} is the difference between a Shannon code length and a shortest program length for $x$. Intuitively, $\delta$ is high whenever the Shannon code derived from $P$ is inefficient at compressing $x$, indicating that $x$ possesses regularities that are atypical of $P$.

The algorithmic information theory literature uses the term \emph{randomness deficiency} to refer to either the universal Martin-L\"of test \labelcref{eq:universalmartinlof} or the universal Levin test \labelcref{eq:universallevin}. Keeping in mind conversions such as \labelcref{eq:ramdascalibration} between the two types of tests, we will develop our theory in terms of the latter.

\begin{example}[Feature Selection through Universal Tests]
Consider an infinite sequence of basis feature functions $(f_i)$, where $\log P$ is expressed as a linear combination of finitely many features: $\log P := \sum_i \alpha_i f_i$. When an observation $x$ exhibits atypical feature values under $P$, it would typically have a substantially higher likelihood under some modified linear combination $\log \tilde{P} := \sum_i \tilde{\alpha}_i f_i$.

If we assume the description of $P$ is given in terms of the coefficient vector $\alpha$, we can bound the Kolmogorov complexity:
\[
K(x\mid P^*)
\stackrel{+}{\leq} K(\tilde{\alpha}\mid\alpha) + \log\frac{1}{\tilde{P}(x)},
\]

This leads to a lower bound on the Levin test:
\[
\delta_P (x)
= - \log P(x)
- K(x\mid P^*)
\stackrel{+}{\geq} \log\frac{\tilde{P} (x)}{P(x)}
- K(\tilde{\alpha}\mid\alpha).
\]

The bound becomes large when the improvement in likelihood (first term) substantially exceeds the complexity cost $K(\tilde{\alpha}\mid\alpha)$ required to modify the coefficients. This demonstrates how the universal tests naturally detect feature-based anomalies: when certain feature statistics of $x$ are unusual under $P$, there exists an alternative distribution $\tilde{P}$ that better explains these feature values, leading to a high value of $\delta_P(x)$.
\end{example}

\section{Decomposition of randomness deficiency}
\label{sec:decomp}

\decompositionpair*
\begin{proof}
By \cref{eq:deltaxy,eq:deltaymidx}, we have:
\begin{align*}
\delta (x, y)
&\ceq - \log P(x,y) - K(x,y \mid (P_{X, Y})^*)
\\&\ceq - \log P(x,y) - K(x \mid (P_{X, Y})^*) - K(y \mid (x, P_{X, Y})^*),
\\\delta (x) + \delta(y \mid x)
&\ceq - \log P(x,y) - K(x \mid (P_{X})^*) - K(y \mid (x,P_{Y\mid X})^*).
\end{align*}
To complete the proof, it suffices to establish the following:
\[
\text{(1) } \quad K(x \mid (P_{X, Y})^*) \ceq K(x \mid (P_{X})^*) \qquad \text{and} \qquad \text{(2) } \quad  K(y \mid (x, P_{X, Y})^*) \ceq K(y \mid (x, P_{Y \mid X})^*)
\]
\paragraph{Proof of (1): } 
Applying \cref{lm:indepinputs}
to our bivariate case yields 
$x \independent P_{Y\mid X} \mid (P_X)^*$,
meaning that $P_{Y\mid X}$ becomes irrelevant when predicting $x$ from a shortest program for $P_X$. Hence,
\[K(x \mid (P_{X, Y})^*)
\ceq K(x \mid (P_{Y\mid X}, P_X)^*)
\ceq K(x \mid (P_X)^*).\]
\paragraph{Proof of (2): } 
Again applying \cref{lm:indepinputs}, $y \indep P_X \mid (x, P_{Y \mid X})^*$,
meaning that $P_{X}$ becomes irrelevant when predicting $y$ from a shortest program for $x$ and $P_{Y\mid X}$. Hence,
\[K(y \mid (x, P_{X, Y})^*)
\ceq K(y \mid (x, P_{Y \mid X}, P_X)^*)
\ceq K(y \mid (x, P_{Y \mid X})^*).\]
\end{proof}
\decomposition*

\begin{proof}
We prove this theorem by induction on $n$.

\noindent\textbf{Base case: }
When $n=1$, we have $\pa_1 = \emptyset$, so the claim is trivial.

\noindent\textbf{Inductive Hypothesis: }
Assume that for any sequence of $n - 1$ strings $x_1, x_2, \ldots, x_{n-1}$,
\[
    \delta(x_1,\dots,x_{n-1}) \ceq \sum_{j=1}^{n-1} \delta (x_j\mid \pa_j).
\]
\textbf{Inductive Step: }
Now we must prove the statement holds for $n$ strings. The causal Markov condition yields
\[P(x_n\mid x_1,\ldots,x_{n-1}) = P(x_n\mid \pa_n).\]
Meanwhile, the algorithmic Markov condition gives us \cref{lm:indepinputs}, which implies
\[
x_n \independent x_1,\ldots,x_{n-1}\mid(\pa_n, P_{X_n\mid \PA_n})^*.
\]
Together with $P_{X_n\mid X_1,\ldots,X_{n-1}} = P_{X_n\mid \PA_n}$, this yields
\[K(x_n \mid (x_1,\ldots,x_{n-1},P_{X_n\mid X_1,\ldots,X_{n-1}})^*)
\ceq K(x_n \mid (\pa_n, P_{X_n \mid \PA_n})^*).\]
Putting these results together,
\begin{align*}
\delta(x_n \mid x_1, \dots, x_{n-1})
&:= - \log P(x_n\mid x_1, \ldots, x_{n-1})
- K(x_n \mid (x_1,\ldots,x_{n-1},P_{X_n\mid X_1,\ldots,X_{n-1}})^*)
\\&\ceq - \log P(x_n\mid \pa_n)
- K(x_n \mid (\pa_n, P_{X_n \mid \PA_n})^*)
\\&= \delta(x_n\mid \pa_n).
\end{align*}
Finally, we combine the inductive hypothesis with \cref{lm:deltaxy}, substituting $x:=(x_1,\dots,x_{n-1})$ and $y:=x_n$:
\begin{align*}
\delta(x_1,\dots,x_n)
&\ceq \delta(x_1,\dots,x_{n-1}) + \delta(x_n \mid x_1, \dots, x_{n-1})
\\&\ceq \sum_{j=1}^n \delta (x_j\mid \pa_j).
\end{align*}
This completes the inductive step and the proof.
\end{proof}

\monotonicity*
\begin{proof}
Corollary 4.1.11 of \cite{gacs2021lecture} states that $\delta$ is non-increasing under marginalization. That is, for all $i$,
\[
\delta (x_i) \cle \delta(x_1,\dots,x_n).
\]
Decomposing the joint randomness deficiency of $x_1, \dots, x_n$ according to \cref{thm:dec},
\begin{align*}
\label{eq:dec}
\delta(x_1,\dots,x_n)
&\ceq \sum_{i=1}^n \delta (x_i\mid \pa_i)
\\&= \delta(x_j\mid \pa_j) + \sum_{i\neq j,\, i=1}^n \delta(x_i \mid \pa_i)
\\&\ceq \delta(x_j\mid \pa_j).
\end{align*}
Putting these results together yields
\begin{equation*}
\delta(x_i)\cle  \delta (x_j\mid \pa_j).
\end{equation*}
\end{proof}

\section{Non-increasingness of Mahalanobis distance}
\label{sec:mahalanobis}

Let $\bx \in \R^n$, and 
let $Q\bx$ denote the projection of $\bx$ onto the subspace spanned by the variables 
$X_{i_1},\dots,X_{i_k}$, where $k < n$. We express the covariance matrix 
$\Sigma_X$ in block matrix form as 
\[
\Sigma_X = \left(\begin{array}{cc} \Sigma_{11} & \Sigma_{12}\\
\Sigma_{21} & \Sigma_{22} \end{array} \right),
\]
with index $1$ referring to the variables $i_1,\dots,i_k$
and $2$ to the remaining variables. Using a known formula for 
inversion of $2\times 2$ block matrices (see
Proposition 3.9.7 in \cite{bernstein2009}), we obtain  
\[
\Sigma^{-1}_X = \left(\begin{array}{cc} \Sigma_{11}^{-1} + \Sigma_{11}^{-1} \Sigma_{12} A \Sigma_{21}\Sigma_{11}^{-1} &   -\Sigma_{11}^{-1} \Sigma_{12} A\\
- A \Sigma_{21} \Sigma_{11}^{-1} & A \end{array} \right),
\]
with $A:=(\Sigma_{22} - \Sigma_{21} \Sigma_{11}^{-1} \Sigma_{12})^{-1}$.
We now compute the difference between the squared Mahalanobis distances 
on $\R^n$ and $\R^k$:
\begin{equation}\label{eq:diffM}
\bx^T \Sigma_{X}^{-1} \bx - (Q\bx)^T \Sigma_{11}^{-1} Q\bx =
\bx^T C \bx, 
\end{equation}
with 
\[
C :=  \left(\begin{array}{cc}  \Sigma_{11}^{-1} \Sigma_{12} A \Sigma_{21}\Sigma_{11}^{-1} &   -\Sigma_{11}^{-1} \Sigma_{12} A\\
- A \Sigma_{21} \Sigma_{11}^{-1} & A \end{array} \right)
= 
\left(\begin{array}{cc}  \Sigma_{11}^{-1} \Sigma_{12} &  0 \\ 0 & -1 \end{array}  \right)
\left(\begin{array}{cc}   A  &    A\\
 A  & A \end{array} \right)
\left(\begin{array}{cc} \Sigma_{21}\Sigma_{11}^{-1}  & 0 \\ 0 & -1 \end{array} \right).
\]
Since $A$ is positive semi-definite, and the rightmost matrix in the product is the transpose of the leftmost matrix, it follows that $C$ is also positive semi-definite. Hence,  
\eqref{eq:diffM} is non-negative.   

\end{document}